\documentclass[letterpaper, 10 pt, conference]{ieeeconf}
\IEEEoverridecommandlockouts  

\usepackage{amsmath,amssymb,bm,mathtools}

\usepackage{amsthm}
\usepackage{booktabs}
\usepackage{multirow}
\usepackage{rotating}
\usepackage{algorithm}
\usepackage[noend]{algpseudocode}
\algrenewcommand\algorithmicrequire{\textbf{Input:}}
\algrenewcommand\algorithmicensure{\textbf{Output:}}
\usepackage{graphicx}
\usepackage[hidelinks]{hyperref}
\usepackage{xcolor}

\newtheorem{theorem}{Theorem}
\newtheorem{lemma}[theorem]{Lemma}
\newtheorem{proposition}[theorem]{Proposition}
\newtheorem{corollary}[theorem]{Corollary}
\theoremstyle{definition}
\newtheorem{definition}[theorem]{Definition}
\newtheorem{assumption}[theorem]{Assumption}

\DeclareMathOperator{\conv}{conv}
\DeclareMathOperator{\spn}{span}

\newcommand{\R}{\mathbb{R}}
\newcommand{\Um}{\mathcal{U}}
\newcommand{\W}{\mathcal{W}}
\newcommand{\F}{\mathcal{F}}
\newcommand{\C}{\mathcal{C}}
\newcommand{\Deltasx}{\Delta}
\newcommand{\Lf}{L_f}
\newcommand{\Lg}[1]{L_{g_{#1}}}
\newcommand{\ip}[2]{\langle #1, #2\rangle}
\newcommand{\bo}{\boldsymbol{o}}

\newcommand{\norm}[1]{\lVert #1\rVert}

\title{Exact Feasibility Certification and Optimal Responsibility Allocation for Multi-Robot CBF Safety Filters}

\author{Chandan Kumar Sah and Jishnu Keshavan%
\thanks{The authors are with the Department of Mechanical Engineering,
Indian Institute of Science, Bengaluru 560012, India.
{\tt\footnotesize chandanks@iisc.ac.in, kjishnu@iisc.ac.in}}}

\begin{document}
\maketitle

\begin{abstract}
\noindent
Multi-robot Control Barrier Function (CBF) safety filters can become infeasible, but a failed quadratic program (QP) does not indicate why the conflict occurred or how to resolve it. To address this, we develop an exact feasibility certificate for multi-agent CBF filters with heterogeneous control-affine dynamics and convex input sets. The certificate quantifies a feasibility reserve by separating the demand imposed by safety constraints from the available actuator supply. This decomposition shows when CBF gain tuning or increased actuation can, and cannot, resolve infeasibility, and identifies the agents and interactions responsible for a conflict. We further propose an algorithm to optimally allocate shared safety constraints by maximizing the worst local feasibility margin, yielding a linear program for polyhedral input sets. In $320$ paired closed-loop simulations, the proposed allocation reduces infeasible control steps from roughly $50\%$ to $6.2\%$, and reduces safety-violating runs from $118/160$ to $24/160$. In addition, across $52$ infeasibility events, the certificate identifies an interaction whose relaxation restores feasibility in $94\%$ of cases. 
\noindent
\href{https://cks0314.github.io/page_multi_CBF_feasibililty/}{\color{red}[Project page]}
\href{https://github.com/cks0314/multi_CBF_feasibililty.git}{\color{red}[Code]\color{black}}

\end{abstract}

\section{Introduction}

Control barrier function (CBF) safety filters provide a principled way to enforce safety while retaining a nominal controller. In multi-robot systems,
however, multiple pairwise or higher-order CBF constraints must often be satisfied simultaneously. Although each constraint may be feasible individually, their conjunction can become infeasible. A standard QP returns only a binary infeasibility flag, giving no indication of which interaction caused the conflict, whether the limitation comes from the CBF encoding or available actuation, or how the conflict should be resolved. Existing approaches address infeasibility through feasibility-guaranteeing constructions
\cite{xiao2022sufficient,xiao2022hocbf}, compatibility conditions \cite{tan2022compat}, nonsmooth barrier compositions \cite{glotfelter2017nonsmooth}, actuation-aware formulations \cite{chen2020guaranteed}, or online adaptation of CBF parameters \cite{xiao2021adaptive}. These methods provide ways to avoid or recover from infeasibility, but do not provide an exact characterization of its source.
Closest to ours, \cite{linfeas2026} gives a Farkas-type feasibility test for stacked barrier constraints in LTI systems, where the normals are constant and
the input set polyhedral, and \cite{underact2026} treats opposing HOCBF bounds for a single underactuated system. Both decide feasibility without quantifying it: neither returns a value that separates the demand a constraint imposes from the actuation available to meet it, which is what supports the diagnosis and allocation below. Neither addresses state-dependent constraint normals, arbitrary convex input sets, or the multi-agent setting.

This paper develops an exact conic certificate that turns multi-robot CBF filter infeasibility into an actionable diagnosis. The certificate exactly
characterizes pointwise feasibility and decomposes the feasibility reserve into constraint demand and actuator supply. This decomposition determines
whether infeasibility can be addressed through the CBF encoding, increased actuation, or redistribution of constraint responsibility. It further shows
that, under the standard HOCBF construction, class-$\mathcal K$ gain tuning cannot increase actuator supply, and characterizes underactuated states where increasing actuation cannot restore feasibility.

\begin{figure}[t]
\centering
\includegraphics[width=\columnwidth]{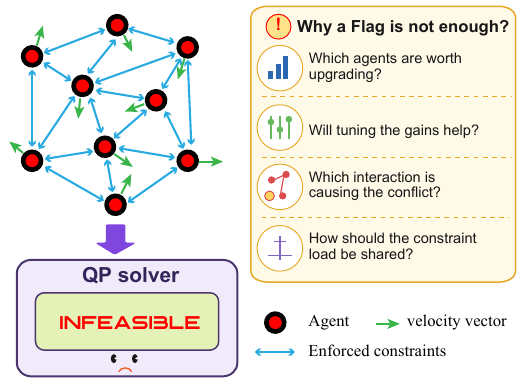}
\caption{An infeasible multi-robot CBF safety filter provides only a binary flag without revealing the cause, the responsible agents, or remedy.}
\end{figure}

The certificate also provides a principled basis for decentralized responsibility allocation. For each shared constraint, responsibility is allocated to maximize the worst local feasibility margin. For polyhedral input sets, this max-min problem reduces to a linear program, directly optimizing local feasibility rather than relying on uniform or capability-weighted heuristics~\cite{cahcbf2026}. The certificate additionally provides sparse conflict attribution through its dual variables and quantifies the marginal value of additional actuation. When no allocation can make all local programs feasible, the same optimization provides the input that maximizes the worst local margin.

Our main contributions are:
\begin{itemize}
    \item We derive an exact conic feasibility certificate that decomposes
    CBF filter feasibility into constraint demand and actuator supply.

    \item We use the certificate to characterize the limits of CBF tuning and
    actuation, identify responsible interactions, and quantify actuator leverage.

    \item We develop a certificate-based responsibility allocation that
    maximizes the worst local feasibility margin.
\end{itemize}
The paper is organized as follows: Section~\ref{sec:prob} formulates the problem; Section~\ref{sec:feas_cert} develops the feasibility certificate; Section~\ref{sec:diag_inf} diagnoses infeasibility; and Section~\ref{sec:alloc} presents the responsibility allocation. Sections~\ref{sec:experiments}–\ref{sec:conc} cover experiments, limitations, and conclusions.

\section{Problem Formulation}
\label{sec:prob}

\subsection{Agents and Constraints}

Consider $N$ agents with heterogeneous control-affine dynamics
\begin{equation}
    \dot{\boldsymbol{x}}_i
    = \boldsymbol{f}_i(\boldsymbol{x}_i)
    + \boldsymbol{g}_i(\boldsymbol{x}_i)\boldsymbol{u}_i,
    \qquad
    \boldsymbol{u}_i\in\mathcal{U}_i\subset\mathbb{R}^{m_i},
\end{equation}
where $\boldsymbol{x}_i\in\mathcal{X}_i\subset\mathbb{R}^{n_i}$, and $\boldsymbol{f}_i,\boldsymbol{g}_i$ are locally Lipschitz and sufficiently smooth. Let
$\boldsymbol{x}=(\boldsymbol{x}_1,\ldots,\boldsymbol{x}_N)$,
$\boldsymbol{u}=(\boldsymbol{u}_1,\ldots,\boldsymbol{u}_N)$, and
$\mathcal{U}=\prod_i\mathcal{U}_i$. For a scalar function $\varphi(\boldsymbol{x})$,
define $\mathcal{L}_f\varphi
=\sum_i\nabla_{\boldsymbol{x}_i}\varphi\,\boldsymbol{f}_i,
\;
\mathcal{L}_{g_i}\varphi
=\nabla_{\boldsymbol{x}_i}\varphi\,\boldsymbol{g}_i
\in\mathbb{R}^{1\times m_i}.$ We index agents by $i$, safety constraints by $k$, and stages of a
barrier chain by $j$. For a nonempty compact convex set $\mathcal{S}\subset\mathbb{R}^m$, its
\emph{support function} is $\sigma_{\mathcal{S}}: \mathbb{R}^m\to\mathbb{R}$,
$\sigma_{\mathcal{S}}(\boldsymbol{\nu}) =\max_{\boldsymbol{s}\in\mathcal{S}}\ip{\boldsymbol{\nu}}{\boldsymbol{s}}$, and is convex. For
$\boldsymbol{\lambda}\in\mathbb{R}^K$ we write
$\operatorname{supp}\boldsymbol{\lambda}=\{k\in\{1,\dots,K\}:\lambda_k\neq0\}$. Let $E$ index a finite collection of safety constraints. Each constraint
$k\in E$ is specified by a barrier function
$h_k:\mathcal{X}\to\mathbb{R}$ with $ h_k(\boldsymbol{x})\ge0.$ Its support $S_k\subseteq\{1,\ldots,N\}$ contains the agents whose states
affect $h_k$, and its arity is $|S_k|$. Thus, $|S_k|=1$, $2$, or $\ge3$ corresponds to single-agent, pairwise, or higher-order interactions, respectively. We write
$E_i=\{k\in E:i\in S_k\}$ for the set of constraints incident to agent~$i$. The joint safe set is $\mathcal{C}
    =\{\boldsymbol{x}\in\mathcal{X}:h_k(\boldsymbol{x})\ge0,\ \forall k\in E\}.$

\begin{assumption}\label{as:sets}
Each $\Um_i$ is nonempty, convex, and compact.
\end{assumption}

\begin{definition}[Uniform relative degree]\label{def:reldeg}
A barrier function $h_k$ has uniform relative degree $r_k \ge 1$ if $\Lg{i}\Lf^{j}h_k \equiv 0$ for all $i\in S_k$ and all $j \le r_k-2$, and
$\Lg{i}\Lf^{r_k-1}h_k \not\equiv 0$ for some $i \in S_k$.
\end{definition}

\subsection{High-Order Barrier Function (HOCBF) Chains}

For higher-order constraints, we enforce safety using a HOCBF chain \cite{xiao2022hocbf}. Following \cite{xiao2022hocbf}, set $\psi_{k,0}=h_k$ and define
\begin{equation}
\psi_{k,j}
=
\dot{\psi}_{k,j-1}+\alpha_{k,j},
\qquad j=1,\ldots,r_k,
\end{equation}
where $\alpha_{k,j}$ are the class-$\mathcal K$ functions. The corresponding admissible set is $\C_\alpha
= \{\boldsymbol{x}:
\psi_{k,j}(\boldsymbol{x})\ge0,\ \forall k\in E,\ 0\le j\le r_k-1\}.$ Under Definition~\ref{def:reldeg}, $\psi_{k,j}$ is independent of the input for $j<r_k$. Consequently, the final stage of the chain is affine in the control and can be written as
\begin{eqnarray}
&\sum_{i\in S_k}
\boldsymbol{G}_{k,i}(\boldsymbol{x})\boldsymbol{u}_i
\geq -b_k(\boldsymbol{x}),
\label{eq:cbf_constraint}\\
&\boldsymbol{G}_{k,i}
:= \Lg{i}\psi_{k,r_k-1},
\;
b_k:=\Lf\psi_{k,r_k-1}+\alpha_{k,r_k}.
\label{eq:Gb}
\end{eqnarray}

Because $h_k$ depends only on the states of agents in $S_k$,
$\boldsymbol{G}_{k,i}\equiv0$ for $i\notin S_k$. Stacking the constraints
therefore gives $\boldsymbol{G}(\boldsymbol{x})\boldsymbol{u}
\ge-\boldsymbol{b}(\boldsymbol{x}),$ where $
\boldsymbol{G}\in
\mathbb{R}^{K\times\sum_i m_i}, $ and $K=|E|.$ The pointwise feasible set is thus
$\F=\{\boldsymbol{x}:
\exists\,\boldsymbol{u}\in\Um,\ 
\boldsymbol{G}(\boldsymbol{x})\boldsymbol{u}
\ge-\boldsymbol{b}(\boldsymbol{x})\}.$ Hence, the safety filter is feasible at $\boldsymbol{x}$ if and only if
$\boldsymbol{x}\in\F$, and $\C_\alpha$ is forward invariant if and only if
$\C_\alpha\subseteq\F$.


\section{An Exact Feasibility Certificate}
\label{sec:feas_cert}

This section characterizes the pointwise feasibility of the safety filter through a scalar reserve $M(\boldsymbol{x})$. To construct it, let
$\boldsymbol{\nu}_{k,i}:=\boldsymbol{G}_{k,i}^{\top}\in\mathbb{R}^{m_i}$ denote the input-space direction associated with constraint $k$ for agent $i$.
For a nonnegative weighting
$\boldsymbol{\lambda}\in\Delta$, where
$\Delta=\{\boldsymbol{\lambda}\in\mathbb{R}^K_{\ge0}:
\sum_k\lambda_k=1\}$, define the combined direction
$\boldsymbol{\nu}_i(\boldsymbol{\lambda})
= \sum_{k:\,i\in S_k}
\lambda_k\boldsymbol{\nu}_{k,i}.$ The maximum input contribution along this direction is given by the support function $\sigma_{\mathcal U_i}(\boldsymbol{\nu}) =
\max_{\boldsymbol{u}\in\mathcal U_i}
\langle\boldsymbol{\nu},\boldsymbol{u}\rangle.$

\begin{definition}[Feasibility reserve]
The \emph{feasibility reserve} at $\boldsymbol{x}$ is $M(\boldsymbol{x})
:=
\min_{\boldsymbol{\lambda}\in\Deltasx}
\left[
\boldsymbol{\lambda}^{\top}\boldsymbol{b}(\boldsymbol{x})
+
\sum_{i=1}^{N}
\sigma_{\Um_i}\!\big(
\boldsymbol{\nu}_i(\boldsymbol{\lambda})
\big)
\right].$ The filter is feasible at $\boldsymbol{x}$ if and only if
$M(\boldsymbol{x})\ge0$.
\end{definition}

In words, the minimization searches for a weighting of the constraints under
which their combined demand exceeds everything the actuators can supply
together. If no such weighting exists, an admissible input does.

\begin{theorem}\label{thm:dual}
Under Assumption~\ref{as:sets}, the minimization defining $M(\boldsymbol{x})$ is a convex program with an attained minimum, and
$\boldsymbol{x}\in\F
\; \Longleftrightarrow \;
M(\boldsymbol{x})\ge0.$
\end{theorem}
\begin{proof}
Let $\phi(\boldsymbol{\lambda})
:=\boldsymbol{\lambda}^{\top}\boldsymbol{b}
+\sigma_{\mathcal U}(\boldsymbol{G}^{\top}\boldsymbol{\lambda})$.
Since $\boldsymbol{G}^{\top}\boldsymbol{\lambda}$ is linear in $\boldsymbol{\lambda}$ and support functions are convex, $\phi$ is convex. It is also continuous because $\mathcal U$ is compact. Since $\Delta$ is
compact, the minimum of $\phi$ over $\Delta$ is attained. We next establish the feasibility equivalence. Define
$\mathcal S:=\{\boldsymbol{G}\boldsymbol{u}+\boldsymbol{b}:
\boldsymbol{u}\in\mathcal U\}$. By convexity and compactness of $\mathcal U$,
$\mathcal S$ is convex and compact. The safety filter is feasible if and only if
$\mathcal S\cap\mathbb R^K_{\ge0}\neq\emptyset$.

Suppose first that $\boldsymbol{x}\notin\mathcal F$. Then
$\mathcal S\cap\mathbb R^K_{\ge0}=\emptyset$. By the strict separation theorem,
there exist $\boldsymbol{\lambda}\neq0$ and $\gamma\in\mathbb R$ such that
$\boldsymbol{\lambda}^{\top}\boldsymbol{z}<\gamma$ for all
$\boldsymbol{z}\in\mathcal S$, and
$\boldsymbol{\lambda}^{\top}\boldsymbol{y}>\gamma$ for all
$\boldsymbol{y}\in\mathbb R^K_{\ge0}$. Setting $\boldsymbol{y}=\boldsymbol{0}$
gives $\gamma<0$. Moreover, if $\lambda_k<0$ for some $k$, then choosing
$\boldsymbol{y}=t\boldsymbol{e}_k$ ($\boldsymbol{e}_k$ denotes the $k$th standard basis vector in $\mathbb{R}^K$) with $t\to\infty$ violates the second
inequality; hence $\boldsymbol{\lambda}\ge0$. Therefore,
$\sup_{\boldsymbol{u}\in\mathcal U}
\boldsymbol{\lambda}^{\top}(\boldsymbol{G}\boldsymbol{u}+\boldsymbol{b})<0$.
By positive homogeneity, normalize $\boldsymbol{\lambda}$ so that
$\boldsymbol{1}^{\top}\boldsymbol{\lambda}=1$. Thus
$\boldsymbol{\lambda}\in\Delta$ and $M(\boldsymbol{x})<0$.

Conversely, suppose that $M(\boldsymbol{x})<0$. Then there exists $\boldsymbol{\lambda}\in\Delta$ such that $\boldsymbol{\lambda}^{\top}\boldsymbol{b}
+\sigma_{\mathcal U}(\boldsymbol{G}^{\top}\boldsymbol{\lambda})<0$.
For any $\boldsymbol{u}\in\mathcal U$,
$\boldsymbol{\lambda}^{\top}(\boldsymbol{G}\boldsymbol{u}+\boldsymbol{b})
\le\boldsymbol{\lambda}^{\top}\boldsymbol{b}
+\sigma_{\mathcal U}(\boldsymbol{G}^{\top}\boldsymbol{\lambda})<0$.
Since $\boldsymbol{\lambda}\ge0$, this is incompatible with
$\boldsymbol{G}\boldsymbol{u}+\boldsymbol{b}\ge0$. Thus
$\boldsymbol{x}\notin\mathcal F$.

Finally, using the product structure $\mathcal U=\prod_i\mathcal U_i$ and the
block decomposition
$\boldsymbol{G}^{\top}\boldsymbol{\lambda}
=(\boldsymbol{\nu}_1(\boldsymbol{\lambda}),\ldots,
\boldsymbol{\nu}_N(\boldsymbol{\lambda}))$, the support function separates as
$\sigma_{\mathcal U}(\boldsymbol{G}^{\top}\boldsymbol{\lambda})
=\sum_{i=1}^N\sigma_{\mathcal U_i}
(\boldsymbol{\nu}_i(\boldsymbol{\lambda}))$.
Substitution gives the stated expression for $M(\boldsymbol{x})$.  
\end{proof} 

The above result follows from separation of a compact convex set, rather than from KKT conditions. The reserve measures \emph{demand} against \emph{supply}. The term $\boldsymbol{\lambda}^{\top} \boldsymbol{b}$ represents the demand imposed by
the constraints, while $\sum_i\sigma_{\Um_i}(\boldsymbol{\nu}_i(\boldsymbol{\lambda}))$ measures the available actuation capability. Feasibility holds when supply meets demand for every weighting $\boldsymbol{\lambda}$. Note that the certificate is instantaneous and does not by itself imply feasibility over a future trajectory.

\begin{lemma}\label{lem:offset}
Suppose $\Um_i=\bo_i+\W_i$, where $\W_i$ is symmetric, convex, and compact.
Then
$\sigma_{\Um_i}(\boldsymbol{\nu})
=\ip{\boldsymbol{\nu}}{\bo_i}+\sigma_{\W_i}(\boldsymbol{\nu})$,
and the feasibility reserve can be written as
$M(\boldsymbol{x})=\min_{\boldsymbol{\lambda}\in\Deltasx}
[\boldsymbol{\lambda}^{\top}\boldsymbol{b}'(\boldsymbol{x})
+S(\boldsymbol{x},\boldsymbol{\lambda})]$, where
\begin{equation}
b'_k=b_k+\sum_{i\in S_k}\ip{\boldsymbol{\nu}_{k,i}}{\bo_i},
\;
S(\boldsymbol{x},\boldsymbol{\lambda})
=\sum_i\sigma_{\W_i}(\boldsymbol{\nu}_i(\boldsymbol{\lambda})).
\end{equation}
\end{lemma}

\begin{proof}
For $\Um_i=\bo_i+\W_i$,
$\sigma_{\Um_i}(\boldsymbol{\nu})
=\max_{\boldsymbol{w}\in\W_i}
\ip{\boldsymbol{\nu}}{\bo_i+\boldsymbol{w}}
=\ip{\boldsymbol{\nu}}{\bo_i}+\sigma_{\W_i}(\boldsymbol{\nu})$.
Substituting this into the definition of $M$ and using
$\boldsymbol{\nu}_i(\boldsymbol{\lambda})
=\sum_{k:i\in S_k}\lambda_k\boldsymbol{\nu}_{k,i}$
combines the offset terms, by Lemma~\ref{lem:offset}, into $\boldsymbol{\lambda}^{\top} \boldsymbol{b}'$ and leaves $S(\boldsymbol{x},\boldsymbol{\lambda})$.  
\end{proof}

This decomposition covers standard input sets such as balls, ellipsoids, boxes, and segments, while asymmetries such as thrust bounds
$[\boldsymbol{u}_{\min},\boldsymbol{u}_{\max}]$ are absorbed into the effective demand $\boldsymbol{b}'$. Define $S^\star(\boldsymbol{x})
:=\min_{\boldsymbol{\lambda}\in\Deltasx}S(\boldsymbol{x},\boldsymbol{\lambda})$ as the supply.

\begin{theorem}\label{thm:invariance}
If $\alpha_{k,j}$ is chain-valued for every $j \le r_k-1$, then for all $i\in S_k$ $\boldsymbol{G}_{k,i}(\boldsymbol{x}) = \Lg{i}\Lf^{\,r_k-1}h_k(\boldsymbol{x}),$ independently of $\alpha_{k,1},\dots,\alpha_{k,r_k}$.
\end{theorem}

\begin{proof}
By Lemma~\ref{lem:closure} (Appendix~A), at $j=r_k-1$,
$\psi_{k,r_k-1}=\Lf^{r_k-1}h_k+\Phi_{k,r_k-1}(h_k,\dots,\Lf^{r_k-2}h_k)$.
Then $\Lg{i}\psi_{k,r_k-1}=\Lg{i}\Lf^{r_k-1}h_k+\sum_{m\le r_k-2}\partial_m\Phi\,\Lg{i}\Lf^{m}h_k$,
and each term in the sum vanishes by Definition~\ref{def:reldeg}. The last
function $\alpha_{k,r_k}$ enters only $b_k$ in \eqref{eq:Gb}.  
\end{proof}

Theorems~\ref{thm:dual} and~\ref{thm:invariance} establish the key separation used throughout the paper. The reserve decomposes into a demand term
$\boldsymbol{\lambda}^\top\boldsymbol{b}$ and a supply term $\sum_i\sigma_{\Um_i}(\boldsymbol{\nu}_i(\boldsymbol{\lambda})))$, where the
latter depends only on the barriers, configuration, and dynamics. Hence, class-$\mathcal{K}$ tuning can modify the demand but cannot change the available
input supply. The achievable feasibility margin is therefore determined by both how the constraint is encoded and by the robots' actuation geometry.

Theorem~\ref{thm:invariance} has one exception: the penultimate class-$\mathcal{K}$ function, $\alpha_{k,r_k-1}$, may depend on states outside its own chain. This affects the constraint geometry rather than the demand term and is not used in our proposed allocation algorithm. We therefore defer it to Appendix~A.

\section{Diagnosing Infeasibility}
\label{sec:diag_inf}
The feasibility certificate (Theorem~\ref{thm:dual}) also enables us to diagnose the source of infeasibility. In particular, the demand-supply decomposition distinguishes failures due to insufficient actuator capability from those caused by conflicting constraints.

\subsection{Actuation limitations}
\label{sec:degen}

This section characterizes the states where increased actuation cannot restore feasibility. Such states cannot arise under full actuation, since a strictly separating
$\boldsymbol{y}_i\in\W_i$ then exists for every constraint and forces $S^\star>0$, but
they can occur for underactuated agents. Throughout, let $\Um_i=\bo_i+\W_i$, where $\W_i$ is symmetric, and consider the scaled sets
$\Um_i^\rho=\bo_i+\rho\W_i$, where $\rho>0$. Increasing $\rho$ increases actuation magnitude while preserving the drift and input directions.

The condition below says that some weighted combination of the constraint
normals points in a direction no agent can push along. Scaling the input sets
then changes nothing, because every support function in the supply is evaluated
at a direction the agent cannot act in.

\begin{theorem}\label{thm:degen}
The supply $S^\star(\boldsymbol{x})=0$ if and only if there exists
$\boldsymbol{\lambda}^\star\in\Deltasx$ such that
\begin{equation}
\label{eq:degen}
\boldsymbol{\nu}_i(\boldsymbol{\lambda}^\star)
\perp \spn\W_i,
\qquad \forall i.
\end{equation}
If $\W_i$ is full-dimensional in $\mathbb{R}^{m_i}$ for every $i$, equivalently
$0\in\operatorname{int}\Um_i$, then \eqref{eq:degen} reduces to $\boldsymbol{\nu}_i (\boldsymbol{\lambda}^\star)=0$ for all $i$. Equivalently, the origin lies in the convex hull of the constraint directions, $0\in\conv\{\boldsymbol{A}\boldsymbol{e}_k:k\in E\},$ where $\boldsymbol{A}:\mathbb{R}^K\to\mathbb{R}^{\sum_i m_i}$ is defined by $\boldsymbol{A} \boldsymbol{\lambda}
= (\boldsymbol{\nu}_i(\boldsymbol{\lambda}))_i$.
\end{theorem}

\begin{proof}
Since $0\in\W_i$ and $\W_i$ is symmetric,
$\sigma_{\W_i}(\boldsymbol{\nu})\ge0$. Moreover,
$\sigma_{\W_i}(\boldsymbol{\nu})=0$ if and only if
$\ip{\boldsymbol{\nu}}{\boldsymbol{w}}=0$ for every
$\boldsymbol{w}\in\W_i$, which is equivalent to
$\boldsymbol{\nu}\perp\spn\W_i$. Hence
$S(\boldsymbol{x},\boldsymbol{\lambda})=0$ if and only if $\boldsymbol{\nu}_i (\boldsymbol{\lambda})\perp\spn\W_i$ for every $i$. Since $\Deltasx$ is compact, the minimum defining $S^\star$ is attained, which proves the first claim. If $\W_i$ is full-dimensional, then $\spn\W_i=\mathbb{R}^{m_i}$, so \eqref{eq:degen} is equivalent to $\boldsymbol{\nu}_i (\boldsymbol{\lambda}^\star)=0$ for every $i$. Thus
$\boldsymbol{A} \boldsymbol{\lambda}^\star=0$ with
$\boldsymbol{\lambda}^\star\in\Deltasx$, which is equivalent to $0\in\conv\{\boldsymbol{A} \boldsymbol{e}_k:k\in E\}$.  
\end{proof}
In words, when \(S^\star(\boldsymbol{x})=0\), some convex combination of the constraint directions is orthogonal to the agents' available actuation subspaces. Consequently, scaling the input sets cannot increase the supply, since the
corresponding support-function contributions are zero
(Corollary~\ref{cor:scale}).

\begin{corollary}\label{cor:scale}
Let $\boldsymbol{\lambda}^\star$ satisfy \eqref{eq:degen}. Then $M^\rho(\boldsymbol{x}) \le\boldsymbol{\lambda}^{\star\top}\boldsymbol{b}'(\boldsymbol{x})$ for every $\rho>0$, where $M^\rho$ is the corresponding reserve for $\Um^\rho$. In particular, if $\boldsymbol{\lambda}^{\star\top} \boldsymbol{b}'(\boldsymbol{x})<0$ then
$\boldsymbol{x}\notin\F^\rho$ for every $\rho>0$, i.e., no scaling of the input sets renders $\boldsymbol{x}$
feasible.
\end{corollary}

\begin{proof}
$\boldsymbol{b}'$ does not depend on $\rho$, and
$\sigma_{\rho\W_i}(\boldsymbol{\nu}_i(\boldsymbol{\lambda}^\star))=\rho\,\sigma_{\W_i}(\boldsymbol{\nu}_i(\boldsymbol{\lambda}^\star))=0$.
Hence $M^\rho(\boldsymbol{x}) \le \boldsymbol{\lambda}^{\star\top}\boldsymbol{b}'+0$, and Theorem~\ref{thm:dual} applies.  
\end{proof}

The degeneracy condition is therefore a property of the collective input geometry, not of the individual input sets being degenerate. Even when each
$\W_i$ is full-dimensional in its own input space, underactuation reduces the dimension of the codomain of $\boldsymbol{A}$, making the condition
$0\in\conv\{\boldsymbol{A}\boldsymbol{e}_k:k\in E\}$ easier to satisfy. Proposition~\ref{prop:ring} exhibits such a state for the agents of
Section~\ref{sec:sim}.

\subsection{Conflict Localization}

This subsection derives three consequences of the certificate, each obtained from the multiplier $\boldsymbol{\lambda}^\star$.


\subsubsection{Marginal Value of Actuation}

Let each agent's reachable set scale independently, $\Um_i^{\rho_i}=\bo_i+\rho_i\W_i$
with $\rho_i>0$, and write $M(\boldsymbol{x};\boldsymbol{\rho})$ for the reserve of the scaled family.

\begin{proposition}\label{prop:leverage}
If the minimizer $\boldsymbol{\lambda}^\star$ is unique, then
\begin{equation}\label{eq:leverage}
\frac{\partial M}{\partial\rho_i}(\boldsymbol{x};\boldsymbol{\rho})
=
\sigma_{\W_i}\!\big(\boldsymbol{\nu}_i(\boldsymbol{\lambda}^\star)\big)
\ge 0.
\end{equation}
In general, when the minimizer is not unique, the Clarke subdifferential~\cite{clarke1990optimization} of $M$ with respect to $\rho_i$ is the convex hull of
$\sigma_{\W_i}\!\big(\boldsymbol{\nu}_i(\boldsymbol{\lambda})\big)$ over all minimizers $\boldsymbol{\lambda}$.
\end{proposition}

\begin{proof}
By Lemma~\ref{lem:offset}, $\phi(\boldsymbol{x},\boldsymbol{\lambda};\boldsymbol{\rho})=\boldsymbol{\lambda}^\top\boldsymbol{b}'
+\sum_i\rho_i\,\sigma_{\W_i}(\boldsymbol{\nu}_i(\boldsymbol{\lambda}))$, in which the offsets enter only
through $\boldsymbol{b}'$ and are independent of $\boldsymbol{\rho}$. So $\phi$ is affine, hence
smooth, in $\boldsymbol{\rho}$, and continuous in $(\boldsymbol{\lambda},\boldsymbol{\rho})$ on the compact set
$\Deltasx$. Danskin's theorem \cite{danskin1966} gives
$\nabla_{\boldsymbol{\rho}}M=\nabla_{\boldsymbol{\rho}}\phi(\boldsymbol{x},\boldsymbol{\lambda}^\star;\boldsymbol{\rho})$, whose $i$-th entry is
$\sigma_{\W_i}(\boldsymbol{\nu}_i (\boldsymbol{\lambda}^\star))$. Nonnegativity follows from $0\in\W_i$.  
\end{proof}

The marginal value of actuation for agent $i$ is the support function of its reachable set evaluated at the aggregate normal at the dual optimum $\boldsymbol{\lambda}^{\star}$. This value is zero for two distinct reasons: the incident normals may cancel, or their aggregate direction may lie outside the agent's actuation directions. In either case, increasing that agent's actuation cannot improve feasibility.

\subsubsection{Sparsity of the Conflict Certificate}

\begin{proposition}\label{prop:sparse}
There exists a minimizer $\boldsymbol{\lambda}^\star$ of $M(\boldsymbol{x})$ with $\lvert\operatorname{supp}\boldsymbol{\lambda}^\star\rvert\ \le\
\operatorname{rank}\!\begin{bmatrix}\boldsymbol{A}\\ \mathbf 1^\top\end{bmatrix}\ \le\
\sum\nolimits_i m_i+1,$ irrespective of $|E|$, where $\boldsymbol{A}\boldsymbol{\lambda}=(\boldsymbol{\nu}_i(\boldsymbol{\lambda}))_i$.
\end{proposition}
 
\begin{proof}
A minimum-support minimizer $\boldsymbol{\lambda}^\star$ exists, the argmin set being
nonempty by Theorem~\ref{thm:dual}; let
$T=\operatorname{supp}(\boldsymbol{\lambda}^\star)$. If $|T|>\operatorname{rank}[\boldsymbol{A};\mathbf 1^\top]$, there exists $\boldsymbol{z}\neq0$ supported on $T$ with $\boldsymbol{A}\boldsymbol{z}=0$ and $\mathbf 1^\top\boldsymbol{z}=0$. Since $\boldsymbol{\lambda}^\star>0$ on $T$, $\boldsymbol{\lambda}^\star+t\boldsymbol{z}\in\Deltasx$ for sufficiently small $|t|$. Moreover, $\boldsymbol{A}\boldsymbol{z}=0$ gives $\phi(\boldsymbol{\lambda}^\star+t\boldsymbol{z})=\phi(\boldsymbol{\lambda}^\star)+t\boldsymbol{b}^\top\boldsymbol{z}$. Minimality for both signs of $t$ implies $\boldsymbol{b}^\top\boldsymbol{z}=0$, so $\phi$ is constant along the feasible line. Since $\mathbf 1^\top\boldsymbol{z}=0$ with $\boldsymbol{z}\neq0$, its entries take both
signs, so moving along this line reaches a point where a support coordinate
vanishes, giving a minimizer of strictly smaller support and contradicting the
choice of $\boldsymbol{\lambda}^\star$. Hence
$|T|\le\operatorname{rank}[\boldsymbol{A};\mathbf 1^\top]$. The second inequality counts rows:
$[\boldsymbol{A};\mathbf 1^\top]$ has $\sum_i m_i+1$ of them, so its rank cannot
exceed that.
\end{proof}

Thus, the worst irreducible conflict contains at most one more constraint than the actuated degrees of freedom, making the bound dimensional rather than combinatorial.


\section{Certificate-Optimal Responsibility Allocation}
\label{sec:alloc}

A decentralized filter must divide each shared constraint among the agents it couples. Existing choices are heuristic
\cite{wang2017safety,cahcbf2026,auction2025,miralloc2026}: they guarantee the individual pairwise conditions but not the feasibility of the local programs
that result when one agent carries several constraints, a gap \cite{cahcbf2026} patches with a clipping safeguard, and \cite{miralloc2026}
assumes away. The proposed certificate prices any division exactly, so the best one follows from an optimization rather than a rule.

The first step is that the certificate separates across agents: if each agent can meet its own assigned share, the joint filter is feasible, and each agent
can verify this using only its own constraints and its own input set.

\begin{theorem}\label{thm:local}
Choose weights $\theta_{k,i}\ge0$ with $\sum_{i\in S_k}\theta_{k,i}=1$, and set
for $\boldsymbol{\mu}\in\Deltasx(E_i)$, $\phi_i(\boldsymbol{\mu})=\sum_{k\in E_i}\boldsymbol{\mu}_k\theta_{k,i}b_k
+\sigma_{\Um_i}\Big(\sum_{k\in E_i}\boldsymbol{\mu}_k\boldsymbol{\nu}_{k,i}\Big),\;
m_i^{\mathrm{loc}}=\min_{\boldsymbol{\mu}}\phi_i(\boldsymbol{\mu}).$ If $m_i^{\mathrm{loc}}(\boldsymbol{x})\ge0$ for every $i$, then $M(\boldsymbol{x})\ge0$ and the filter is
feasible at $\boldsymbol{x}$. Each $\phi_i$ depends only on the constraints incident to agent $i$ and on $\Um_i$.
\end{theorem}

\begin{proof}
Splitting $\boldsymbol{b}$ and using
$\sum_{i\in S_k}\theta_{k,i}=1$ gives $\boldsymbol{\lambda}^\top\boldsymbol{b}
=\sum_k\lambda_k b_k\sum_{i\in S_k}\theta_{k,i}
=\sum_i\sum_{k\in E_i}\lambda_k\theta_{k,i}b_k,$
so the objective separates as
$\phi(\boldsymbol{\lambda})=\sum_i\phi_i(\boldsymbol{\lambda}_{E_i})$.
Each $\phi_i$ is positively homogeneous. For
$\boldsymbol{\lambda}\in\Delta$, let
$w_i=\sum_{k\in E_i}\lambda_k$. If $w_i>0$, then
$\boldsymbol{\lambda}_{E_i}/w_i\in\Delta(E_i)$ and
$\phi_i(\boldsymbol{\lambda}_{E_i})
=w_i\phi_i\!\left({\boldsymbol{\lambda}_{E_i}}/{w_i}\right)
\ge w_i m_i^{\mathrm{loc}}\ge0.$
If $w_i=0$, then $\phi_i(\boldsymbol{\lambda}_{E_i})=0$. Hence
$\phi(\boldsymbol{\lambda})\ge0$ for every $\boldsymbol{\lambda}\in\Delta$,
and therefore $M(\boldsymbol{x})\ge0$.  
\end{proof}

\begin{proposition}\label{prop:allocsafe}
Let $\boldsymbol{\theta}$ satisfy $\theta_{k,i}\ge0$ and $\sum_{i\in S_k}\theta_{k,i}=1$ for every $k$, and suppose each agent applies an input $\boldsymbol{u}_i\in\Um_i$ meeting its own share,
\begin{equation}\label{eq:localcon}
\theta_{k,i}b_k(\boldsymbol{x})+\boldsymbol{G}_{k,i}(\boldsymbol{x})\boldsymbol{u}_i\ \ge\ 0,
\qquad k\in E_i .
\end{equation}
Then $\psi_{k,r_k}(\boldsymbol{x},\boldsymbol{u})\ge0$ for every $k\in E$. Consequently the HOCBF forward-invariance argument \cite{xiao2022hocbf} applies for \emph{every} admissible $\boldsymbol{\theta}$, including one recomputed at each state.
\end{proposition}

\begin{proof}
Summing \eqref{eq:localcon} over $i\in S_k$ and using
$\sum_{i\in S_k}\theta_{k,i}=1$ gives $b_k+\sum_{i\in S_k}\boldsymbol{G}_{k,i}\boldsymbol{u}_i\ge0$, which is exactly $\psi_{k,r_k}\ge0$. The argument is pointwise in $\boldsymbol{x}$ and does not involve
$\boldsymbol{\theta}$ beyond the normalization, so it is unaffected by any dependence of
$\boldsymbol{\theta}$ on the state or on time.  
\end{proof}

Proposition~\ref{prop:allocsafe} separates safety from local feasibility: any valid allocation preserves the barrier condition, so $\theta$ can be optimized without compromising safety. Its effect is instead on whether the local programs \eqref{eq:localcon} remain feasible, which we address next.

The uniform choice $\theta_{k,i}=1/|S_k|$ recovers the standard decentralized CBF condition of \cite{wang2017safety}, used as the hand-crafted baseline in \cite{zhang2025gcbfplus}. Theorem~\ref{thm:local} shows that this is one choice within a broader family, making $\theta$ a design variable. The per-agent quantity in Theorem~\ref{thm:local} is not merely a bound. It is exactly the largest worst-case slack agent $i$ can reach with its own input.

\begin{lemma}\label{lem:localexact}
For $\boldsymbol{\theta}$ admissible and $\Um_i$ convex compact,
\begin{equation}\label{eq:localmargin}
m_i^{\mathrm{loc}}(\boldsymbol{x};\boldsymbol{\theta})
=\max_{\boldsymbol{u}_i\in\Um_i}\ \min_{k\in E_i}\
\big(\theta_{k,i}b_k(\boldsymbol{x})+\boldsymbol{G}_{k,i}(\boldsymbol{x})\boldsymbol{u}_i\big),
\end{equation}
so $m_i^{\mathrm{loc}}\ge0$ if and only if agent $i$'s local program admits a
solution.
\end{lemma}

\begin{proof}
Writing $\sigma_{\Um_i}$ as a maximum over $\Um_i$,
$\phi_i(\boldsymbol{\mu})=\max_{\boldsymbol{u}_i\in\Um_i}\sum_{k\in E_i}\mu_k(\theta_{k,i} b_k+\boldsymbol{G}_{k,i}\boldsymbol{u}_i)$.
The expression is bilinear in $(\boldsymbol{\mu}, \boldsymbol{u}_i)$ and both $\Deltasx(E_i)$ and
$\Um_i$ are convex and compact, so Sion's theorem \cite{sion1958} exchanges the minimum and the maximum. For fixed $\boldsymbol{u}_i$, the inner objective is linear in $\boldsymbol{\mu}$ and its minimum over the simplex is attained at a vertex, giving
$\min_{k\in E_i}$. The stated equivalence is then immediate, since agent $i$'s program is exactly the requirement that some $\boldsymbol{u}_i\in\Um_i$
satisfy every $\theta_{k,i} b_k+\boldsymbol{G}_{k,i} \boldsymbol{u}_i\ge0$.  
\end{proof}

Lemma~\ref{lem:localexact} sharpens Theorem~\ref{thm:local} by identifying each agent's quantity as its exact feasibility margin. Thus, maximizing the minimum local margin $m_i^{\mathrm{loc}}$ is the natural allocation problem and is tractable. Maximizing the worst margin looks like a search over allocations. Dualizing the inner minimization collapses it to one linear program.

\begin{theorem}\label{thm:alloc}
Let each $\Um_i$ be a polytope. Then
\begin{equation}\label{eq:allocLP}
\max_{\boldsymbol{\theta}}\ \min_i m_i^{\mathrm{loc}}
(\boldsymbol{x};\boldsymbol{\theta})
\end{equation}
equals the optimal value $t^\star$ of the linear program
\begin{equation}\label{eq:LP}
\begin{aligned}
\max_{t,\boldsymbol{\theta},\boldsymbol{u}}\quad & t\\
\text{s.t.}\quad
& t\le\theta_{k,i}b_k+\boldsymbol{G}_{k,i}\boldsymbol{u}_i, \quad k\in E,\ i\in S_k,\\
& \sum_{i\in S_k}\theta_{k,i}=1,\quad
\boldsymbol{\theta}\ge0,\quad
\boldsymbol{u}_i\in\Um_i .
\end{aligned}
\end{equation}
If $t^\star\ge0$, the resulting allocation makes all local programs feasible and, by Theorem~\ref{thm:local}, certifies $M(\boldsymbol{x})\ge0$. If $t^\star<0$, no admissible allocation can make all local programs feasible. The converse does not hold, i.e., $M(\boldsymbol{x})\ge0$
may still have $t^\star<0$, since a jointly feasible input need not admit a
feasible decomposition into local shares. Thus,
$M(\boldsymbol{x})\ge0>t^\star$ characterizes the feasibility gap introduced
by decentralization.
\end{theorem}

\begin{proof}
Substitute \eqref{eq:localmargin} into \eqref{eq:allocLP}. The result is a
maximization over $\boldsymbol{\theta}$ and the per-agent inputs $\boldsymbol{u}_i$ of the
minimum over $(k,i)$ of terms affine in those variables. Introducing the
epigraph variable $t$ gives \eqref{eq:LP}. Its constraints are linear because
each $\Um_i$ is a polytope and $\theta_{k,i}b_k$ is linear in
$\boldsymbol{\theta}$. The consequences follow from Lemma~\ref{lem:localexact} and
Theorem~\ref{thm:local}.  
\end{proof}

\begin{algorithm}[t]
\small
\caption{Certificate-optimal responsibility allocation}
\label{alg:alloc}
\begin{algorithmic}[1]
\Require state $\boldsymbol{x}$; barriers $\{h_k\}$ with supports $\{S_k\}$; input sets
$\{\Um_i\}$; nominal inputs $\boldsymbol{u}^{\mathrm{nom}}$
\Ensure inputs $\boldsymbol{u}$ applied by the agents
\State form $\boldsymbol{G}(\boldsymbol{x})$, $\boldsymbol{b}(\boldsymbol{x})$ from the chains \eqref{eq:Gb}
       \Comment{Thm.~\ref{thm:invariance}}
\State $(t^\star,\boldsymbol{\theta}^\star,\boldsymbol{u}^\star)\gets$ solve \eqref{eq:LP}
       
\If{$t^\star\ge0$} 
  \State broadcast $\boldsymbol{\theta}^\star$; discard $\boldsymbol{u}^\star$
  \ForAll{agents $i$ \textbf{in parallel}} 
    \State $\boldsymbol{u}_i\gets\arg\min\,\lVert \boldsymbol{u}_i-\boldsymbol{u}^{\mathrm{nom}}_i\rVert^2$
           over $\boldsymbol{u}_i\in\Um_i$
    \State \hphantom{$\boldsymbol{u}_i\gets$}s.t.
           $\theta^\star_{k,i}b_k+\boldsymbol{G}_{k,i}\boldsymbol{u}_i\ge0$ for $k\in E_i$
  \EndFor
\Else \Comment{no split can work}
  \State $\boldsymbol{u}\gets\boldsymbol{u}^\star$
         \Comment{maximizes the worst margin over all agents}
  \State report $\operatorname{supp}\boldsymbol{\lambda}^\star$ and the agents with
         $\sigma_{\W_i}(\boldsymbol{\nu}_i(\boldsymbol{\lambda}^\star))>0$
\EndIf
\end{algorithmic}
\end{algorithm}

Algorithm~\ref{alg:alloc} is the control step that Theorem~\ref{thm:alloc}
yields. The linear program is solved on the joint state, and only the scalars $\boldsymbol{\theta}^\star$ are broadcast.
Each agent then solves a small program over its own input and its own constraints, so execution is local even though the allocation is not, as in
\cite{zhang2025gcbfplus,miralloc2026}.
 
The program also returns inputs $\boldsymbol{u}^\star$, which serve two purposes. When $t^\star\ge0$, we discard them, because they maximize the margin rather than
track the task; their only role is to show that a feasible input exists, and each agent instead picks the input closest to its nominal one that meets its
share. When $t^\star<0$ no split works, and $\boldsymbol{u}^\star$ is the input that keeps
the worst margin as high as possible. Applying it is then the least infeasible action available, at no extra cost.
 
Lemma~\ref{lem:localexact} holds for any convex compact $\Um_i$, but \eqref{eq:LP} is linear only when each $\Um_i$ is a polytope; otherwise, the same argument gives a convex program that is still solvable. 

\begin{figure*}[t]
\centering
\includegraphics[width=\textwidth]{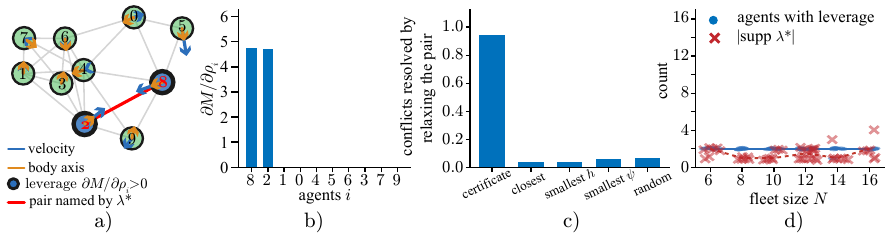}
\caption{Certificate-based conflict localization. (a) A representative conflict and the pair identified by $\operatorname{supp}\boldsymbol{\lambda}^\star$. 
(b) Agent-wise actuation leverage $\partial {M} / \partial {\rho_i}$. (c) Conflicts resolved by relaxing the identified pair versus alternative criteria. 
(d) Sparsity of $\boldsymbol{\lambda}^\star$ with fleet size.}
\label{fig:attribution}
\end{figure*}

\section{Numerical Results}
\label{sec:experiments}

We evaluate the effectiveness of the proposed attribution and allocation algorithms in closed-loop simulations.

\subsection{Simulation Setup}
\label{sec:sim}
\label{sec:setup}

\emph{Agents.} Each agent is a planar underactuated vehicle with state
$\boldsymbol{x}_i=(\boldsymbol{q}_i,\boldsymbol{v}_i,\phi_i)\in\R^5$, comprising position, velocity, and
heading, and a single scalar input $u_i$, so $m_i=1$:
\begin{equation}\label{eq:simdyn}
\dot{\boldsymbol{q}}_i=\boldsymbol{v}_i,\;
\dot{\boldsymbol{v}}_i=u_i\begin{bmatrix}\cos\phi_i & \sin\phi_i\end{bmatrix}^{\top},\;
\dot\phi_i=\omega_i .
\end{equation}
By \eqref{eq:simdyn} an agent accelerates only along its heading, which is itself a state and turns at a bounded rate, $|\omega_i|\le6$\,rad\,s$^{-1}$, with first-order tracking of a commanded heading at time constant $\tau=0.12$\,s. The input sets are the intervals $\Um_i=[-\bar u_i,\bar u_i]$, with $\bar u_i=1$ for all agents in the identical-fleet case and $\bar u_i$ drawn uniformly from $[0.5,2.0]$, a fourfold spread, in the heterogeneous case.

\emph{Constraints.} Every pair within a sensing radius of $2.4$\,m is enforced, giving $h_k=\norm{\boldsymbol{q}_i-\boldsymbol{q}_j}^2-(2r)^2$ with agent radius $r=0.3$\,m. Each $h_k$ has relative degree two, and the chain uses constant gains $\alpha_{k,1}=3h_k$ and $\alpha_{k,2}=3\psi_{k,1}$.

\emph{Task.} Agents start at random positions in a $5.2\times5.2$\,m square, separated by at least $0.75$\,m, with zero initial velocity and headings
pointing at their goals, where goals are drawn independently in the same square. Because start and goal assignments are independent, trajectories cross, and the
fleet is driven into dense conflict rather than a structured formation. The nominal controller is a PD law, $\boldsymbol{a}^{\mathrm{nom}}_i{=} {-}3(\boldsymbol{q}_i{-}\boldsymbol{q}^{\mathrm{goal}}_i)-3.6\boldsymbol{v}_i$. Each run is integrated at $50$\,Hz, with the filter and the allocation recomputed at every step.

\subsection{Localization of Conflicts}

We evaluate whether the dual multiplier identifies the interactions responsible for infeasibility. Across randomized tasks with $N=6,\ldots,16$ underactuated
agents, we obtain $52$ conflict events spanning $8$-$78$ enforced pairs. All events have positive supply and are therefore not degenerate (Theorem~\ref{thm:degen}). The dual support provides a finer diagnosis by identifying the interactions consuming the available actuation.

Fig.~\ref{fig:attribution}(a) shows a conflict with $N=10$ and $24$ enforced pairs. The dual support is a single pair, and relaxing it moves the reserve from $-10.02$ to $+0.72$, whereas relaxing the closest pair, the smallest
barrier value, or the smallest $\psi$ does not. Over all $52$ conflicts the pair named by $\operatorname{supp}\boldsymbol{\lambda}^\star$ restores feasibility in $94\%$ of cases against $4$-$6\%$ for those heuristics (Fig.~\ref{fig:attribution}(c)). The three failures admit a second certificate once the named pair is relaxed, so the attribution gives a sufficient cause, not a unique one. The support stays small as the fleet
grows, averaging $1.4$ pairs with exactly two agents carrying nonzero leverage \eqref{eq:leverage}, well inside the bound of Proposition~\ref{prop:sparse}
(Fig.~\ref{fig:attribution}(d)).

\begin{figure}[t]
\centering
\includegraphics[width=\columnwidth]{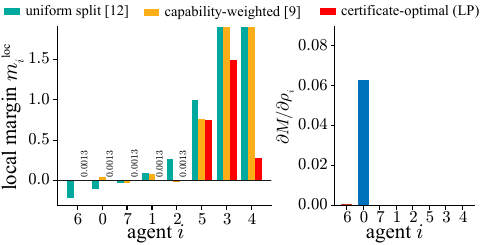}
\caption{Responsibility allocation for $N=8$ and $25$ enforced pairs. (a) Local feasibility margins under uniform~\cite{wang2017safety}, capability-weighted~\cite{cahcbf2026}, and certificate-optimal allocation (b) Agent-wise capability leverage $\partial M/\partial\rho_i$.}
\label{fig:allocmech}
\end{figure}
\subsection{Allocation Mechanism}

Figure~\ref{fig:allocmech} illustrates the allocation mechanism at a fixed state. For $N=8$ and $25$ enforced pairs, three agents have no admissible input
under the uniform split and two under the capability-weighted split. The proposed optimizer redistributes responsibility from agents with slack to
those with negative margins, increasing the worst local margin from $-0.211$ to $+0.001$, so that all local programs become feasible. This illustrates the
max-min objective in \eqref{eq:LP}. Fig.~\ref{fig:allocmech}(b) shows the
leverage $\partial M/\partial\rho_i$ at the same state. It is identical under all three allocations, because the joint certificate does not depend on $\boldsymbol{\theta}$, and it is nonzero for a single agent. The allocation, therefore, changes neither the fleet's capability nor which agent holds it, only whether
the local programs can use it.

\begin{figure*}[t]
\centering
\includegraphics[width=0.95\textwidth]{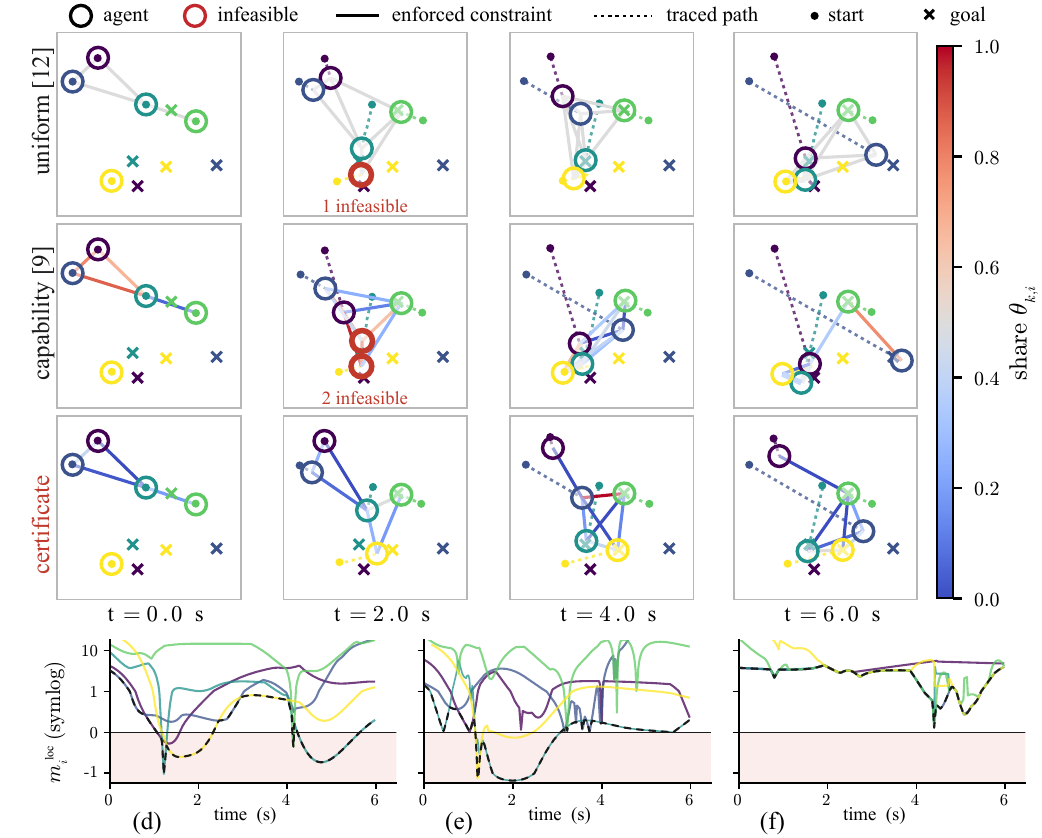}
\caption{Comparison of three responsibility allocations for an identical task instance ($N=5$), with fixed initial conditions, goals, and actuation bounds. Top: four snapshots of each run; paths and goals are color-coded by agent, while edge color indicates the responsibility share $\theta_{k,i}$. Red outlines mark agents with infeasible local programs. (d)-(f) Local margins $m_i^{\mathrm{loc}}$; black dashed line shows the worst margin.}
\label{fig:story}
\end{figure*}
\begin{figure}[t]
\centering
\includegraphics[width=\columnwidth]{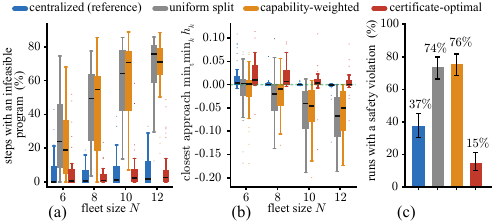}
\caption{Responsibility allocation in a decentralized filter on heterogeneous fleets, over $160$ paired runs with identical tasks for all methods. (a) Percentage of control steps at which a program is infeasible
(b) The resulting closest approach. (c) The percentage of the safety distance that is violated.}
\label{fig:alloc}
\end{figure}

\begin{table}[t]
\centering
\caption{Infeasible control steps (\%) over $160$ paired runs per fleet.}
\label{tab:alloc}
\small
\setlength{\tabcolsep}{3.5pt}
\begin{tabular}{llccccc}
\toprule
& allocation & $N{=}6$ & $N{=}8$ & $N{=}10$ & $N{=}12$ & pooled\\
\midrule
\multirow{4}{*}{\rotatebox{90}{\scriptsize identical}}
 & centralized (ref.) & 5.9 & 5.0 & 8.3 & 8.0 & 6.8\\
 & uniform \cite{wang2017safety} & 27.3 & 44.9 & 60.9 & 68.6 & 50.4\\
 & capability \cite{cahcbf2026} & 25.8 & 39.9 & 59.0 & 68.2 & 48.2\\
 & \textbf{certificate \eqref{eq:LP}} & \textbf{5.2} & \textbf{4.8} & \textbf{7.0} & \textbf{7.2} & \textbf{6.0}\\
\midrule
\multirow{4}{*}{\rotatebox{90}{\scriptsize heterog.}}
 & centralized (ref.) & 5.7 & 3.9 & 5.7 & 7.0 & 5.6\\
 & uniform \cite{wang2017safety} & 28.9 & 46.8 & 59.7 & 68.4 & 51.0\\
 & capability \cite{cahcbf2026} & 23.9 & 43.9 & 61.2 & 68.0 & 49.2\\
 & \textbf{certificate \eqref{eq:LP}} & \textbf{6.3} & \textbf{3.1} & \textbf{7.5} & \textbf{7.8} & \textbf{6.2}\\
\bottomrule
\end{tabular}
\end{table}

\subsection{Closed-Loop Evaluation of the Allocation}

We evaluate the decentralized filter of Section~\ref{sec:setup}, where each agent enforces only its allocated share $\theta_{k,i}b_k$ of each shared constraint. We compare uniform allocation \cite{wang2017safety}, capability-weighted allocation \cite{cahcbf2026}, and the proposed optimizer
\eqref{eq:LP}. A centralized filter is included as a reference. Each method is evaluated on the same $160$ paired random tasks for each fleet type, with
$N\in\{6,8,10,12\}$.

Figure~\ref{fig:story} illustrates a representative run under the three allocations. With $N=5$ and identical initial conditions, local program
infeasibility occurs at $45\%$ and $32\%$ of control steps under uniform and capability-weighted allocation, respectively, compared with only $0.3\%$ under the optimizer of \eqref{eq:LP}. All three runs impose the same safety constraints, and only their allocation among agents differs.

Table~\ref{tab:alloc} and Fig.~\ref{fig:alloc} show that the
certificate-optimal allocation substantially reduces local infeasibility. The
infeasible-step rate is $6.0\%$ on identical fleets and $6.2\%$ on
heterogeneous ones, against roughly $50\%$ for both heuristics and $6.8\%$ and
$5.6\%$ for the centralized reference. The proposed allocation therefore
recovers most of the feasibility of a centralized filter, while both
heuristics retain a large decentralization penalty. Safety follows: on
heterogeneous fleets the mean closest approach is $+0.017$ against $-0.044$
and $-0.038$ for the heuristics, and violations occur in $24/160$ runs against
$118/160$ and $121/160$.

We additionally evaluate the $t^\star<0$ branch of Algorithm~\ref{alg:alloc}, where no responsibility allocation can make all local programs feasible.
Across $87$ paired runs, the proposed coordinated input reduces infeasible steps from $6.1\%$ to $4.3\%$ and improves the mean closest approach from
$+0.016$ to $+0.020$, relative to independent per-agent relaxation. On the $24$ runs in which this branch is activated, the corresponding infeasible-step
rates are $10.7\%$ and $17.3\%$. These results demonstrate that the proposed allocation improves both local feasibility and the resulting closed-loop
safety, while capability weighting provides little benefit over uniform allocation.

\section{Limitations}
\label{sec:lim}
The allocation guarantees local feasibility, not task completion. Deadlock and liveness are therefore outside the scope of this work. The local certificate of Theorem~\ref{thm:local} is sufficient but not necessary, and although we establish that degenerate states exist, their closed-loop reachability is not characterized. Hardware validation is left for future work.

\section{Conclusion}
\label{sec:conc}

We presented an exact conic certificate for multi-robot CBF filter
infeasibility that separates constraint demand from actuator supply. It
determines whether a conflict can be resolved by gain tuning, by more
actuation, or only by redividing responsibility, and it exposes degenerate
underactuated states where extra actuation cannot help at all. Its dual
variables localize the conflict to a small set of interactions and grade the
agents by the marginal value of actuation. Because the certificate prices any
division of a shared constraint, the optimal division follows from a linear
program, which in closed loop recovers most of the feasibility of a
centralized filter and reduces safety violations well below either heuristic.
Together these turn an infeasibility flag into an actionable diagnosis.

\bibliographystyle{ieeetr}
\bibliography{citation.bib}

\section*{Appendix}

\subsection*{A. Chain closure and the encoding lever}
Lemma~\ref{lem:closure} is the induction behind
Theorem~\ref{thm:invariance}; Theorem~\ref{thm:lever} identifies the single
exception, in which the penultimate class-$\mathcal K$ function depends on
states outside its own chain and does reshape the input coefficients.

\begin{lemma}\label{lem:closure}
Fix $k$ and suppose $\alpha_{k,l}$ is chain-valued for every $l\le j$, with
$j\le r_k-1$. Then $\psi_{k,j}$ depends only on
$(h_k,\Lf h_k,\ldots,\Lf^j h_k)$. In particular, there exists a differentiable
$\Phi_{k,j}$ such that
\begin{equation}
\label{eq:closure}
\psi_{k,j}
=
\Lf^j h_k
+
\Phi_{k,j}\!\left(h_k,\Lf h_k,\ldots,\Lf^{j-1}h_k\right),
\end{equation}
with $\Phi_{k,0}\equiv0$. Moreover, $\Lg{i}\psi_{k,j}\equiv0$ for all
$i\in S_k$ whenever $j\le r_k-2$. Thus, the input can first appear at
$j=r_k-1$, exactly as prescribed by the relative-degree assumption.
\end{lemma}

\begin{proof}
Induction on $j$. At $j=0$ the claim is $\psi_{k,0}=h_k$ with
$\Lg{i}h_k=0$. Assume it at $j-1\le r_k-2$. Differentiating,
$\dot\psi_{k,j-1}=\Lf\psi_{k,j-1}+\sum_i\Lg{i}\psi_{k,j-1}\boldsymbol{u}_i$, and every
Lie derivative in $\Lg{i}\psi_{k,j-1}$ has order at most $j-1\le r_k-2$ and so
vanishes by Definition~\ref{def:reldeg}. Hence
$\dot\psi_{k,j-1}=\Lf\psi_{k,j-1}=\Lf^{j}h_k+\widetilde\Phi_{k,j}$ for a
differentiable $\widetilde\Phi_{k,j}$ of the lower derivatives. Since
$\alpha_{k,j}$ is chain-valued it depends only on those same derivatives, and
absorbing it into $\widetilde\Phi_{k,j}$ gives \eqref{eq:closure}. Stage
$j=r_k-1$ is the first at which $\Lg{i}\psi_{k,j}$ need not vanish.  
\end{proof}

\begin{theorem}\label{thm:lever}
Let $\alpha_{k,l}$ be chain-valued for $l\le r_k-2$, and
\begin{equation}\label{eq:mult}
\alpha_{k,r_k-1}(\boldsymbol{x}) \;{=}\; a_k(\boldsymbol{x})\,\psi_{k,r_k-2}(\boldsymbol{x}),\qquad a_k \in C^1,
\end{equation}
with $r_k \ge 2$. Then $\psi_{k,r_k-1}$ is input-free, the terminal condition remains affine in $\boldsymbol{u}$, and
\begin{equation}\label{eq:lever}
\boldsymbol{G}_{k,i}(\boldsymbol{x}) \;{=}\; \underbrace{\Lg{i}\Lf^{\,r_k{-}1}h_k(\boldsymbol{x})}_{\text{invariant part}}
\;{+}\;\underbrace{\psi_{k,r_k{-}2}(\boldsymbol{x})\,\Lg{i}a_k(\boldsymbol{x})}_{\text{design part}}.
\end{equation}
\end{theorem}

\begin{proof}
Write $\psi := \psi_{k,r_k-2}$, so $\psi_{k,r_k-1}=\Lf\psi + a_k\psi$ (using
$\Lg{i}\psi=0$ from Lemma~\ref{lem:closure}). This is a state function, so $\psi_{k,r_k-1}$ is input-free
and $\psi_{k,r_k}$ is affine in $\boldsymbol{u}$. Differentiating with the product rule, $\Lg{i}\psi_{k,r_k-1}=\Lg{i}(\Lf\psi)+\psi\,\Lg{i}a_k+a_k\,\Lg{i}\psi.$ The third term vanishes. For the first, by \eqref{eq:closure},
$\Lf\psi=\Lf^{r_k-1}h_k+\sum_{m\le r_k-3}\partial_m\Phi_{k,r_k-2}\Lf^{m+1}h_k$, whose $\Lg{i}$ is $\Lg{i}\Lf^{r_k-1}h_k$ because all remaining indices are $\le r_k-2$.  
\end{proof}

Thus, the design can modify the supply only through \eqref{eq:lever}. Three properties are notable: if $\Lg{i}a_k$ is not parallel to $\Lg{i}\Lf^{r_k-1}h_k$, it rotates the constraint gradient; its magnitude
scales with $\psi_{k,r_k-2}$ and vanishes on the corresponding boundary; and it requires $a_k$ to depend on states outside its own chain. If $a_k$ is
chain-valued, Lemma~\ref{lem:closure} gives $\Lg{i}a_k=0$, reducing \eqref{eq:lever} to Theorem~\ref{thm:invariance}.

The choice \eqref{eq:mult} scales the design term with $\psi_{k,r_k-2}$, so it offers freedom before a
conflict rather than during one, and it requires $a_k$ to depend on states outside its own chain: a chain-valued $a_k$ gives $\Lg{i}a_k=0$ and recovers
Theorem~\ref{thm:invariance}. Any bounded, strictly positive $a_k$ preserves forward invariance~\cite{khalil2002}, so the lever carries no safety cost, but it changes both $\mathcal{C}_\alpha$ and $\mathcal{F}$.

\subsection*{B. When degeneracy can and cannot occur}
For the agents of Section~\ref{sec:sim}, degeneracy is not vacuous.
\begin{proposition}\label{prop:ring}
Place $N$ agents at the vertices of a regular polygon,
$\boldsymbol{q}_i=R\hat{\boldsymbol{\rho}}_i$, each at rest with its body axis tangential, $\hat{\boldsymbol{e}}_i=\hat{\boldsymbol{\tau}}_i$, and let every pair be enforced. Then $\boldsymbol{\nu}_i(\boldsymbol{\lambda})=0$ for every agent at the uniform multiplier $\lambda_k=1/|E|$, so $S^\star(\boldsymbol{x})=0$ and $\boldsymbol{x}$ is degenerate. 
\end{proposition}

\begin{proof}
Here $\boldsymbol{G}_{k,i}=2\ip{\boldsymbol{p}_{ij}}{\hat{\boldsymbol{e}}_i}$ with
$\boldsymbol{p}_{ij}=\boldsymbol{q}_i-\boldsymbol{q}_j$. Summing over the pairs incident to agent $i$
and using $\sum_j\boldsymbol{q}_j=0$ gives $\sum_{j\ne i}\boldsymbol{p}_{ij}=N\boldsymbol{q}_i$, which is
radial, whereas $\hat{\boldsymbol{e}}_i$ is tangential. Hence $\boldsymbol{\nu}_i(\boldsymbol{\lambda})=0$
and $S(\boldsymbol{x},\boldsymbol{\lambda})=0$.
\end{proof}

\end{document}